\documentclass[letterpaper, 10 pt, conference]{ieeeconf}  

\IEEEoverridecommandlockouts                              
\usepackage{cite}

\usepackage{amsmath,amssymb,amsfonts,amsthm}
\usepackage{graphicx}
\usepackage{textcomp}
\usepackage{xcolor}
\let\labelindent\relax
\usepackage{enumitem}
\usepackage{subcaption}
\usepackage{tikz}
\usepackage{float}
\usepackage{mathrsfs}

\newtheoremstyle{bolddefinition}
  {}{}{\normalfont}{}{\bfseries}{.}{ }{}
\theoremstyle{bolddefinition}
\newtheorem{theorem}{\textbf{Theorem}}
\newtheorem{definition}{\textbf{Definition}}
\newtheorem{problem}{\textbf{Problem}}

\newtheorem{proposition}{\textbf{Proposition}}
\newtheorem{remark}{\textbf{Remark}}

\usepackage{glossaries}
\newacronym{cbf}{CBF}{Control Barrier Function}
\newacronym{clf}{CLF}{Control Lyapunov Function}
\makeglossaries

\usepackage{algorithm}
\usepackage[noend]{algpseudocode}

\usepackage{blindtext}
\usepackage{pgfplots}
\pgfplotsset{compat=1.18}

\newcommand{\figfontsize}{\small}

\DeclareMathOperator*{\argmin}{argmin}

\def\BibTeX{{\rm B\kern-.05em{\sc i\kern-.025em b}\kern-.08em
    T\kern-.1667em\lower.7ex\hbox{E}\kern-.125emX}}

\title{Conflict-Aware Switching for CBF-CLF-Based Multi-Goal Navigation}

\author{Rohan Walia and Kevin Leahy
\thanks{R. Walia and K. Leahy are with the Robotics Engineering Department, Worcester Polytechnic Institute, Worcester, MA 01609, USA. 
        {\tt\small \{rwalia,kleahy\}@wpi.edu}. This work was supported by funding from Amazon Robotics.}%
}
\begin{document}
\algrenewcommand\algorithmiccomment[1]{\hskip0.5em$\triangleright$ #1}

\maketitle

\begin{abstract}
Quadratic programs (QPs) using Control Barrier Functions (CBFs) and Control Lyapunov Functions (CLFs) are widely used for safe control in reach-and-avoid navigation. However, the inherently conflicting nature of CBF and CLF constraints can lead to performance degradation, including slowdowns and deadlocks. This issue is exacerbated in multi-goal scenarios, where multiple nominal control objectives must be satisfied under shared safety constraints. Existing approaches for preemptive safety are often computationally expensive or overly conservative, while methods that relax or switch between nominal objectives are not well-suited for sequential goal-to-goal navigation. To address these limitations, we propose a conflict-aware switching strategy that detects high-conflict conditions and switches between available nominal control objectives to reduce constraint conflict. We apply this approach to multi-agent, multi-goal reach-and-avoid scenarios under CBF-CLF-QP control. Compared to a baseline sequential goal traversal strategy, our method reduces both completion time and timeout rates, demonstrating improved performance in satisfying all nominal control objectives while respecting safety constraints. 
\end{abstract}


\section{Introduction}

Safe navigation tasks, such as reach-and-avoid scenarios and waypoint navigation in crowded environments, require balancing a nominal control objective against safety constraints. Control Barrier Functions (CBFs) and Control Lyapunov Functions (CLFs) have emerged as a widely adopted framework in the controls and robotics community, providing formal guarantees for safety constraint satisfaction while 
preserving progress toward nominal control objectives. By design, CBFs offer a minimally invasive corrective action when implemented via quadratic programs (QPs) \cite{cbf_theory_and_applications}. However, the feasible set of control inputs decreases as the degree of conflict between nominal and safety constraints increases. To an extent, this problem is solved by relaxing the nominal control objective locally to prioritize safety \cite{cbf_theory_and_applications}. Still, operating under conflicting control constraints can significantly degrade system performance, manifesting in the form of slowdowns, and eventually deadlocks.

Slowdowns that are not accounted for by motion planning \cite{reactive_motion_planning} or explicit recovery behavior modeling may lead to unpredictable delays and failure of mission-critical tasks. Forward propagation of the system model could help detect and avoid slowdowns by predicting proximity to high-conflict regions. For example, \cite{mpc_cbf_discrete} integrates a discrete CBF constraint into a Model Predictive Control (MPC) framework, achieving ``predictive safety" by preemptively avoiding unsafe regions. However, such approaches can become computationally expensive.

Alternatively, the geometric interaction between CBF and CLF constraints along system trajectories can be exploited for computationally efficient preemptive deadlock avoidance, without requiring forward propagation of the system model. Such techniques leverage the property that CBF-CLF QPs are inherently prone to \textit{asymptotically stable undesired equilibria} \cite{cbf_clf_deadlock_detection, cbf_clf_deadlock_detection_stronger_guarantees}, commonly known as deadlocks. As shown in \cite{cbf_clf_deadlock_detection}, one solution is to eliminate all undesired equilibria entirely. However, \cite{cbf_clf_deadlock_avoidance_decentrallized_multi_agent} 
demonstrates that this can be overly conservative. Altering the system trajectory to avoid every known deadlock may degrade performance by diverting the system from its nominal control objective, particularly when many deadlocks are never encountered in practice.

In the presence of multiple nominal control objectives (e.g., goal-to-goal navigation in warehouse robotics applications), temporarily conflicted regions can be avoided by switching to an alternative objective located in a less constrained region. In \cite{dimitra_constraint_selection}, Lee et al. propose an optimization strategy that maximizes feasibility under a given set of predetermined hard and soft constraints. Constraints are prioritized according to feasibility, with the least feasible ``soft'' constraints temporarily relaxed to ensure that the remaining constraints can be satisfied at each timestep. However, this strategy enforces the simultaneous satisfaction of the most feasible constraints at any given time, which differs fundamentally from the sequential nature of goal-to-goal navigation. Additionally, this strategy does not guarantee that  once disregarded ``soft" constraints will be satisfied eventually.


Event-triggered scheduling allows switching between nominal control objectives one at a time. In \cite{event_triggered_scheduling}, Tabuada et al. use a feedback controller that switches to an alternative control objective based on a thresholded input error. However, this approach relies on predetermined execution times for each control task, and primarily supports switching between only two nominal control objectives. Although it can be extended to accommodate more objectives, it requires at least one constraint to take precedence over others.

These limitations highlight the need for a priority-agnostic switching strategy that favors conflict reduction in the presence of multiple nominal control objectives. In this work, we formulate a switching strategy for multi-agent-multi-goal reach-and-avoid scenarios implemented through CLF-CBF QPs. By adopting this strategy, we empirically demonstrate reduction in timeouts and average completion time for satisfying all control objectives for each agent. Our specific contributions include:

\begin{enumerate}
    \item Identifying a heuristic that captures degree of conflict for a given pair of safety and nominal control objective constraints; and
    \item Integration of this strategy for multi-agent-multi-goal reach-and-avoid scenarios; and 
    \item Empirically demonstrating improvement in performance of multi-agent-multi-goal reach-and-avoid scenarios using switching by virtue of timeout percentage and completion time.
\end{enumerate}

\section{Problem Formulation}

We use the following notation throughout the paper. $\mathbb{R}$ and $\mathbb{R}_+$ represent the set of real and non-negative real numbers respectively. $\mathbf{x}$ denotes a vector. A function $\phi: \mathbb{R} \rightarrow \mathbb{R}$ is an extended class-$\mathcal{K_{\infty}}$ function if $\phi(0)=0$ and $\phi$ is strictly increasing on the interval $(-\infty,\infty)$. The interior and boundary of a closed set $\mathcal{S}$ are denoted as $\mathrm{Int}(\mathcal{S})$ and $\partial \mathcal{S}$ respectively. $\nabla V$ represents the gradient of a function $V$ with respect to state $\mathbf{x}$. The $r$-th order Lie derivative of a continuously differentiable function $V:\mathbb{R}^n \mapsto \mathbb{R}$ along a vector field $f: \mathbb{R}^n \mapsto \mathbb{R}^n$ at a point $x \in \mathbb{R}^n$ is denoted by $L^r_fV(x) \triangleq \frac{\partial^r V}{\partial x^r} f(x)$, where $r$ is omitted if $r=1$. $\mathbf{1}_n \in \mathbb{R}^n$ denotes a vector of ones, and $I_n \in \mathbb{R}^{n \times n}$ denotes an identity matrix. For vector $v \in \mathbb{R}^n$, we define the orthogonal projection matrix onto the subspace perpendicular to $v$ as $\mathcal{P}_v = \|v\|^2 I_n - vv^T \in \mathbb{R}^{n \times n}$. $\mathfrak{g}$ denotes a goal location.

\subsection{CBF-CLF QPs}

Consider a control-affine system of the form:

\begin{equation}
    \dot{\mathbf{x}} = f(\mathbf{x}) + g(\mathbf{x})\mathbf{u},
    \label{eq:control_affine_system}
\end{equation}
where $\mathbf{x} \in \mathbb R^n$ is the state, $\mathbf{u} \in \mathcal{U} \subseteq \mathbb R^m$ is the control input, and functions $f: \mathbb R^n \mapsto \mathbb R^n$ and $g: \mathbb R^n \mapsto \mathbb R^{n \times m}$ are the known, locally Lipschitz drift vector and control matrix respectively. Let $h: \mathbb R^n \mapsto \mathbb R$ be a continuously differentiable function. A safe set $\mathcal{C} \subset \mathbb R^n$ is defined as the zero super-level set of $h$ such that
\begin{subequations}
\begin{align}
    \mathcal{C} &= \{\mathbf{x} \in \mathbb R^n \mid h(\mathbf{x}) \geq 0\}, \\
    \partial \mathcal{C} &= \{\mathbf{x} \in \mathbb R^n \mid h(\mathbf{x}) = 0\}.
\end{align}
\label{eq:safe_set}
\end{subequations}
\begin{definition}[Forward Invariance]
    Set $\mathcal{C}$ is said to be \emph{forward invariant} with respect to system \eqref{eq:control_affine_system} if for every $\mathbf{x}_0 \in \mathcal{C}$, $\mathbf{x}(t) \in \mathcal{C} \;\forall t \geq 0$ where $\mathbf{x}(0) = \mathbf{x}_0$. 
\end{definition}

\begin{definition}[Control Barrier Function]
    Given a set $\mathcal{C}$ defined by \eqref{eq:safe_set} for a continuously differentiable function $h: \mathbb R^n \mapsto \mathbb{R}$ with $0$ a regular value, $h$ is a \emph{\acrfull{cbf}} for system \eqref{eq:control_affine_system} with respect to $\mathcal{C}$ if there exists a locally Lipschitz function $\phi \in \mathcal{K}_{\infty}$ such that, $\forall \mathbf{x} \in \mathbb R^n$,
    \begin{equation}\label{eq:cbf_condition}
        \sup_{\mathbf{u} \in \mathcal{U}}L_fh(\mathbf{x}) + L_gh(\mathbf{x})\mathbf{u} \geq -\phi\left(h(\mathbf{x})\right).
    \end{equation}
\end{definition}

\begin{definition} [Control Lyapunov Function]
    for \eqref{eq:control_affine_system}, a positive-definite function $V$ is defined as a Control Lyapunov Function if it satisfies: 
    \begin{equation}
        \inf_{\mathbf{u} \in \mathcal{U}} \; L_f V(\mathbf{x}) + L_g V(\mathbf{x}) \mathbf{u}
        \leq -\gamma(V(\mathbf{x})\big),
        \label{eq:clf_condition_slack}
    \end{equation}
    where $\gamma$ is a class $\mathcal{K}$ function.
\end{definition}

\begin{theorem} [\hspace{-0.3pt}\cite{cbf_theory_and_applications}, Thm. 2]
    Given a set $\mathcal{C}$ as defined in \eqref{eq:safe_set} for a continuously differentiable function $h: \mathbb R^n \mapsto \mathbb R$, any locally Lipschitz continuous controller $\mathbf{u}$ that satisfies \eqref{eq:cbf_condition} will render $\mathcal{C}$ forward invariant. 
    \label{thm:cbf_forward_invariance}
\end{theorem}
Reach-and-avoid problems can be implemented for a CLF and CBF by solving for the optimum control input $\mathbf{u(x)}$ that respects the CLF and CBF constraints through the following quadratic program:

\begin{equation}
\begin{aligned}
    \mathbf{u(x)} = \argmin_{(\mathbf{u}, \delta) \in \mathbb{R}^{m+1}} \quad 
    & \tfrac{1}{2} \mathbf{u}^T \mathbf{u} + p\,\delta^2 \\[4pt]
    \text{s.t.} \quad 
    & L_f V(\mathbf{x}) + L_g V(\mathbf{x}) \mathbf{u} \le -\alpha\big(V(\mathbf{x})\big) + \delta \\[4pt]
    & L_f h(\mathbf{x}) + L_g h(\mathbf{x}) \mathbf{u} \ge -\phi\big(h(\mathbf{x})\big) \\[4pt]
    & \delta \ge 0.
    \label{eq:cbf_clf_qp}
\end{aligned}
\end{equation}
Here $\delta$ is a slack variable which relaxes the CLF constraint (nominal control objective) to prioritize the safety constraint imposed by the CBF, and $p$ is a slack penalty that penalizes excessive relaxation. 

\subsection{Multi-objective control}

Consider a set of $k$ distinct goals implemented through a set of $k$ equal-priority control objectives, defined as:
\begin{equation}
    \begin{aligned}
        \mathcal{G}(t) &= \left\{ \mathfrak{g}_i \in \mathbb{R}^n \mid i \in \{1, \dots, k\} \right\},\\
        \mathcal{V}(t) &= \left\{ V_i : \mathcal{X} \rightarrow \mathbb{R}_{\geq 0} \mid i \in \{1, \dots, k\} \right\},
        \label{eq:goal_constraint_set}
    \end{aligned}
\end{equation}
where $\mathcal{G}(t)$ and $\mathcal{V}(t)$ are the set of remaining goals and CLF constraints at time $t$, respectively, and $\mathcal{X} = \mathbb{R}^n$ represents the state space. Each CLF $V_i \in \mathcal{V}(t)$ is constructed with respect to its corresponding goal $\mathfrak{g}_i \in \mathcal{G}(t)$, such that $V_i(x) = 0 \iff x = \mathfrak{g}_i$. For a multi-obstacle setting, we define a set of obstacles $\mathcal{O} = \{o_1, \dots, o_r\}$ where $o_j \in \mathbb{R}^n$, and a corresponding set of CBF functions $\mathcal{H} = \{h_j : \mathcal{X} \rightarrow \mathbb{R},\ j \in \{1, \dots, r\}\}$, where each $h_j$ encodes a safety constraint with respect to obstacle $o_j$. For goal-to-goal navigation, it is not feasible to satisfy all nominal control objectives simultaneously. Therefore, we satisfy the \textit{engaged} nominal constraint $V^* \in \mathcal{V}(t)$ and enforce safety with respect to all obstacles simultaneously through a QP of the form:
\begin{equation}
    \begin{aligned}
        \mathbf{u(x)} = \argmin_{(\mathbf{u}, \delta) \in \mathbb{R}^{m+1}} \quad 
        & \tfrac{1}{2} \mathbf{u}^T \mathbf{u} + p\delta^2 \\[4pt]
        \text{s.t.} \quad 
        & L_f h_j + L_g h_j\, \mathbf{u} \ge -\phi(h_j), \quad \forall h_j \in \mathcal{H} \\[4pt]
        & L_f V^* + L_g V^*\, \mathbf{u} \le -\alpha(V^*) + \delta \\[4pt]
        & \delta \ge 0
        \label{eq:ncbf_nclf_qp}
    \end{aligned}
\end{equation}
A trivial switching condition simply satisfies the engaged nominal constraint $V^*$ as:
\begin{equation}
    \|\mathbf{x} - \mathfrak{g}^*\|^2 \leq \epsilon_\mathfrak{g}
    \label{eq:goal_constraint_satisfaction}
\end{equation}
for some common goal satisfaction threshold $\epsilon_\mathfrak{g} \in \mathbb{R^+}$. A trivial switching policy $\pi: \mathcal{X} \times 2^{\mathcal{V}} \rightarrow \mathfrak{g}_{next}$ simply selects the next goal $g_{next}$ in a conflict-agnostic manner as:
\begin{equation}
    \pi_{\text{seq}}(\mathbf{x}, \mathcal{V}(t)) = \argmin_{i \in \mathcal{I}(t)} \; i
    \label{eq:policy_next_clf}
\end{equation}
where $\mathcal{I}(t) = \{i : V_i \in \mathcal{V}(t)\}$ is the set indices of the remaining nominal objectives. Once goal $\mathfrak{g}_i$ is reached, goal set $\mathcal{G}$ and constraint set $\mathcal{V}$ are updated as:
\begin{equation}
    \begin{aligned}
        \mathcal{G}(t^+) &= \mathcal{G}(t) \setminus \{\mathfrak{g}_i\} \\
        \mathcal{V}(t^+) &= \mathcal{V}(t) \setminus \{V_i\}.
    \end{aligned} 
    \label{eq:goal_constraint_set_update}
\end{equation}
This process is repeated until the following \emph{terminal state} is achieved:
\begin{equation}
\begin{aligned}
    \mathcal{G} &= \emptyset \\
    \mathcal{V} &= \emptyset.
\end{aligned}
\label{eq:multi_goal_reach_and_avoid_terminal_state}
\end{equation}

\subsection{Conflict-aware switching}
Existence of multiple goals motivates switching to an alternative nominal control objective that is more compatible with a given safety constraint. In order to do that, we need to identify when a given pair of CBF-CLF constraints are most incompatible. Consider the Lagrangian of QP \eqref{eq:ncbf_nclf_qp}:
\begin{equation}
    \mathcal{L}(\mathbf{u}, \delta) = \mathcal{J}(\mathbf{u}, \delta) 
    + \lambda_1 \mathcal{C}_1(\mathbf{u}, \delta)
    + \lambda_2 \mathcal{C}_2(\mathbf{u})
    \label{eq:lagrangian}
\end{equation}
where
\begin{align}
    \mathcal{J}(\mathbf{u}, \delta) &= \frac{1}{2}\mathbf{u}^T\mathbf{u} + p\delta^2 \\[4pt]
    \mathcal{C}_1(\mathbf{u}, \delta) &= L_f V^* + L_g V^*\,\mathbf{u} + \alpha(V^*) - \delta \\[4pt]
    \mathcal{C}_2(\mathbf{u}) &= -\left(L_f \mathbf{h} + L_g \mathbf{h}\,\mathbf{u} + \boldsymbol{\phi}(\mathbf{h})\right).
\end{align}
From the KKT conditions for \eqref{eq:lagrangian} \cite{cbf_clf_deadlock_detection}:
\begin{equation}
     \frac{\partial \mathcal{L}}{\partial \mathbf{u}} = 
    \frac{\partial \mathcal{J}}{\partial \mathbf{u}} + 
    \lambda_1 \frac{\partial \mathcal{C}_1}{\partial \mathbf{u}} + 
    \lambda_2 \frac{\partial \mathcal{C}_2}{\partial \mathbf{u}} = 0
    \label{eq:kkt_condition}
\end{equation}
where:
\begin{equation}
    \frac{\partial \mathcal{C}_1}{\partial \mathbf{u}} = L_g V^*, \; \frac{\partial \mathcal{C}_2}{\partial \mathbf{u}} = -L_g \mathbf{h}.
\end{equation}
At a deadlock ($u = 0$), $\mathcal{C}_1$ and $\mathcal{C}_2$ oppose each other. Therefore, their gradients would be anti-parallel:
\begin{align}
     \frac{\partial \mathcal{C}_1}{\partial \mathbf{u}} &= - \frac{\partial \mathcal{C}_2}{\partial \mathbf{u}} \\
     \implies  L_g V^* &=  -(-L_g \mathbf{h}) \\
     \implies (\nabla V &- \nabla h) g = 0
    \label{eq:anti_parallel_constraints}
\end{align}
From \eqref{eq:anti_parallel_constraints}, the following holds true for non-zero values of $L_gV$ and $L_gh$:
\begin{equation}
    L_g V^* =  L_g \mathbf{h} \implies \nabla V^* = \nabla h .
    \label{eq:sufficient_deadlock_condition}
\end{equation}
\begin{remark}
Condition \eqref{eq:sufficient_deadlock_condition} is sufficient but not necessary for detecting deadlocks. For control-affine system \eqref{eq:control_affine_system}, a necessary 
condition for deadlock detection is established in Theorem 1, Eq. (6) of \cite{cbf_clf_deadlock_detection} as:
\begin{equation}
    f(x) = \lambda_1 G \nabla V - \lambda_2 G\nabla h,
    \label{eq:necessary_deadlock_condition}
\end{equation}
where $G = gg^T$. However, $\lambda_1 G \nabla V$ and $\lambda_2 G \nabla h$ vanish when the relative degree between the system dynamics and the CLF or CBF is greater than one \cite{HOCLFs, ECBFs}. Thus, deadlock detection becomes impossible.
\label{rm:necessary_deadlock_condition}
\end{remark}
Following \eqref{eq:sufficient_deadlock_condition} and Remark \ref{rm:necessary_deadlock_condition}, we define a conflict as follows:
\begin{definition}[Constraint Conflict]
The conflict between constraints $\mathcal{C}_1$ and $\mathcal{C}_2$ at state $x$ is the proximity of their gradients $\nabla V$ and $\nabla h$ to the following condition: 
\begin{equation}
     \nabla V = \nabla h
    \label{eq:anti_parallel_condition}
\end{equation}
\end{definition}
That is, conflict is high when $\nabla V$ and $\nabla h$ point in same 
direction, forcing $\mathcal{C}_1$ and $\mathcal{C}_2$ to demand opposing 
corrections. A conflict heuristic which measures the proximity to this condition can be defined as follows:

\begin{definition}[Conflict Heuristic]
    A conflict heuristic $\Phi(x, V, h) :  \mathcal{X} \times \mathcal{V} \times \mathcal{H} \rightarrow \mathbb{R}$ is a function that:
    \begin{itemize}
        \item Attains its maximum value $C$ when the constraint gradients are aligned, i.e., 
        $\nabla V = \nabla h$, corresponding to maximum conflict (deadlock)
        \item Returns a strictly smaller value otherwise, decreasing monotonically as the constraint gradients diverge from alignment
    \end{itemize}
\end{definition}

To this end, we define our main problem as follows: 
\begin{problem}
    Given $k$ distinct equal-priority goals $\mathcal{G}(0) = \{\mathfrak{g}_1, \dots, \mathfrak{g}_k\}$, corresponding control objectives $\mathcal{V}(0) = \{V_1, \dots, V_k\}$, and a control barrier function $h \in \mathcal{H}(t)$, find a heuristic function $\phi: \mathcal{X} \times \mathcal{V} \times \mathcal{H} \rightarrow \mathbb{R}$ such that:
    \begin{equation}
        \mathcal{V}^*(t) = \pi(\mathbf{x}, \mathcal{V}(t))\big|_{h} = \argmin_{V_i \in \mathcal{V}(t)} \phi(\mathbf{x}, V_i, h),
        \label{eq:policy_min_phi}
    \end{equation}
    where $\mathcal{V}^*(t)$ is the set of least-conflicted constraints at time $t$, and:
    \begin{equation}
        i^* = \argmin_{i \in \mathcal{I}(t)} \mathcal{V}^*(t),
        \label{eq:policy_min_clf}
    \end{equation}
    where $i^*$ is the index of least-conflicted, minimum-valued nominal constraint objective at time $t$.
    \label{prob:main}
\end{problem}

\section{Methodology}

In this section, we discuss the limitations of an existing alignment-based heuristic that measures conflict between a nominal control objective and a safety constraint. We then formalize the notion of a consistent conflict heuristic and show that any such heuristic guarantees finding the least-conflicted set of goals per \eqref{eq:policy_min_phi}. Finally, we present a new alignment-based heuristic that satisfies this consistency criterion.

\subsection{Alignment-based conflict heuristic}

In \cite{cbf_clf_deadlock_detection}, Reis et al. prove that for CBF-CLF QPs of the form \eqref{eq:cbf_clf_qp}, asymptotically stable equilibria exist within the interior and on the boundary of the safe set. The interior equilibria are desired, as they drive the system towards completion of the nominal control objective (in this case, ensuring $V \approx 0$). Boundary equilibria occur where both CLF and CBF constraints are active. Since this represents a deadlock,  boundary equilibria are undesired. Following \eqref{eq:necessary_deadlock_condition}, Reis et al. define a collinearity metric $\mathcal{D}:\mathcal{X} \times \mathcal{SO}(n) \to \mathbb{R}_{\ge 0}$ to track proximity to a deadlock using (eq 24, \cite{cbf_clf_deadlock_detection}) as:
\begin{equation}
    \mathcal{D}(\mathbf{x}, Q) = \frac{1}{2} \nabla V(\mathbf{x}, Q)^T G(P_f + P_{G\nabla h}) G \nabla V(\mathbf{x}, Q)
    \label{eq:reis_collinearity_heuristic_original}
\end{equation}
where $P_f$ and $P_{G\nabla h}$ represent scaled orthogonal projections. Here $Q \in \mathcal{SO}(n)$ is a rotation state for the Lyapunov function $V$, designed to eliminate boundary equilibria by reshaping the original Lyapunov function $V$. As mentioned earlier, eliminating all undesired equilibria has been shown to be an overly conservative approach \cite{cbf_clf_deadlock_avoidance_decentrallized_multi_agent}. To use \eqref{eq:reis_collinearity_heuristic_original} as a conflict heuristic without modifying the original CLF constraint, we set $Q = I_n$ to obtain:
\begin{equation}
    \mathcal{D}(\mathbf{x}) = \frac{1}{2} \nabla V(\mathbf{x})^T G(P_f + P_{G\nabla h}) G \nabla V(\mathbf{x}).
    \label{eq:reis_collinearity_heuristic}
\end{equation}
To demonstrate the need for this heuristic, consider a 2D single integrator driftless dynamics model ($f(\mathbf{x}) = 0$ in \eqref{eq:control_affine_system}) with identity control matrix g. Then, \eqref{eq:reis_collinearity_heuristic} becomes


\begin{equation}
    \mathcal{D}(\mathbf{x}) = \frac{1}{2} \nabla V(\mathbf{x})^T P_{\nabla h} \nabla V(\mathbf{x}).
    \label{eq:D_simplified}
\end{equation}
Expanding scaled orthogonal projection $P_{\nabla h}$:

\begin{equation}
    \mathcal{D}(\mathbf{x}) = \frac{1}{2} \nabla V(\mathbf{x})^T\left( \frac{\nabla h(\mathbf{x}) \nabla h(\mathbf{x})^T}{\|\nabla h(\mathbf{x})\|^2} \right) \nabla V(\mathbf{x}).
    \label{eq:D_expanded}
\end{equation}
This heuristic helps identify when $\nabla h$ and $\nabla V$ are collinear. However, due to its positive semi-definite nature, it does not differentiate between the best case ($\nabla h$ and $\nabla V$ are antiparallel) and worst case ($\nabla h$ and $\nabla V$ are parallel). This is shown in Fig.  \ref{fig:collinearity_cos_sim_dif_collinearity}.
\begin{figure} [h]
    \centering
    \includegraphics[width=1\linewidth]{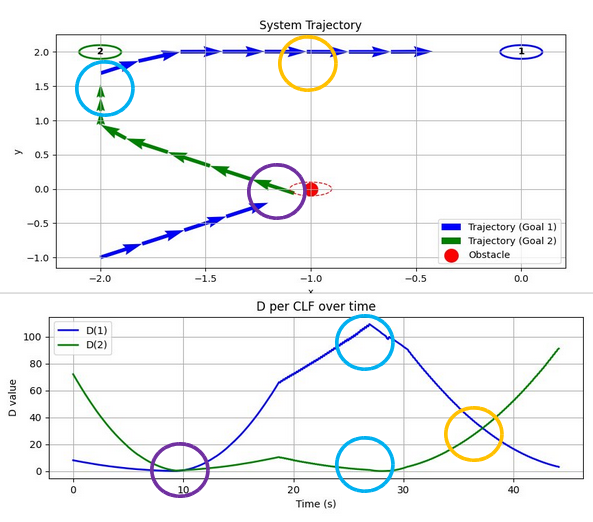}
    \caption{\textbf{Inconsistent alignment-based heuristic.} An agent is tasked with reaching goals 1 and 2, switching goals based on a thresholded value of $\mathcal{D}$ as defined in \eqref{eq:D_expanded}. Trajectories and the corresponding plot of $\mathcal{D}$ are color-coded by the engaged goal. The purple and blue annotations highlight when the system encounters the obstacle and reaches goal $\mathbf{g}_2$, respectively. Notice that $\mathcal{D}$ fails to distinguish between a head-on obstacle approach and an obstacle-free goal approach, as both produce near-zero $\mathcal{D}$ values. This inconsistency motivates the need for a consistent alignment-based heuristic $\phi_\theta$ 
\eqref{eq:cos_sim}.}
    \label{fig:collinearity_cos_sim_dif_collinearity}
\end{figure}

To find the least conflicted goal,  a heuristic $\phi$ needs to rank all available nominal constraints by establishing partial order on set $\mathcal{V}(t)$ for a given control barrier function $h$ at state $x$. This requires the heuristic to be \emph{consistent} in the following sense:

\begin{definition}[Consistent Conflict heuristic]
Given sets $\mathcal{X}$, $\mathcal{V}$ as defined in \eqref{eq:goal_constraint_set} and 
$\mathcal{H}$ as defined in \eqref{eq:ncbf_nclf_qp}, a function 
$\phi : \mathcal{X} \times \mathcal{V} \times \mathcal{H} \to \mathbb{R}$ is a consistent 
conflict heuristic if, for any $V_1, V_2 \in \mathcal{V}$ and $h \in \mathcal{H}$,
\begin{equation}
    \phi(\mathbf{x}, V_1, h) \leq \phi(\mathbf{x}, V_2, h)
    \label{eq:consistent_heuristic}
\end{equation}
iff $(V_1, h)$ is no more conflicted than $(V_2, h)$ at $x$.
\label{def:consistent_heuristic}
\end{definition}
If a consistent heuristic is available, the following theorem guarantees finding the least-conflicted goal at time $t$:

\begin{theorem}[Optimal and Complete Goal-Selection Policy]
Let $\mathcal{V}(t)$ be a non-empty finite set of control objectives and let
$\phi : \mathcal{X} \times \mathcal{V} \times \mathcal{H} \to \mathbb{R}$ be a consistent
conflict heuristic. Then the policy
\[
    \pi(\mathbf{x}, \mathcal{V}(t))\big|_{h} = \argmin_{V_i \in \mathcal{V}(t)} \phi(\mathbf{x}, V_i, h)
\]
selects the least-conflicted objective with respect to $h \in \mathcal{H}$ at $x \in \mathcal{X}$.
\label{thm:optimal_policy}
\end{theorem}

\begin{proof}
Since $\mathcal{V}(t)$ is non-empty and finite, and $\phi$ maps into $\mathbb{R}$, the minimum
value of $\phi$ over $\mathcal{V}(t)$ is attained by at least one element. Since $\phi$ is a
consistent conflict heuristic, by \eqref{eq:consistent_heuristic}, the minimizer of $\phi$
corresponds to an objective that is no more conflicted than any other $V_i \in \mathcal{V}(t)$.
Hence $\pi$ selects a least-conflicted objective if one exists.
\end{proof}

\subsection{Cosine-Similarity based heuristic}

We now introduce an alternative alignment-based heuristic $\phi_{\theta} : \mathcal{X} \times \mathcal{V} \times \mathcal{H} \to [-1, 1]$ that measures conflict by checking the cosine similarity between $\nabla V$ and $\nabla h$ as:

\begin{equation}
    \phi_{\theta}(\mathbf{x}, V(\mathbf{x}), h(\mathbf{x})) = \frac{\nabla V(\mathbf{x})^T \nabla h(\mathbf{x})}{\|\nabla V(\mathbf{x})\| \|\nabla h(\mathbf{x})\|}.
    \label{eq:cos_sim}
\end{equation}
As shown in Fig. \ref{fig:collinearity_cos_sim_diff_cos_sim}, this heuristic is able to differentiate between the best and worst case scenarios, thus imposing a partial order on the constraint set $\mathcal{V}(t)$. A higher value of $\phi_{\theta}$ would correspond to a higher degree of conflict between the engaged CLF and CBF, indicating that a slowdown (and potentially, a deadlock) is imminent. 

\begin{figure}[h]
    \centering
    \includegraphics[width=1\linewidth]{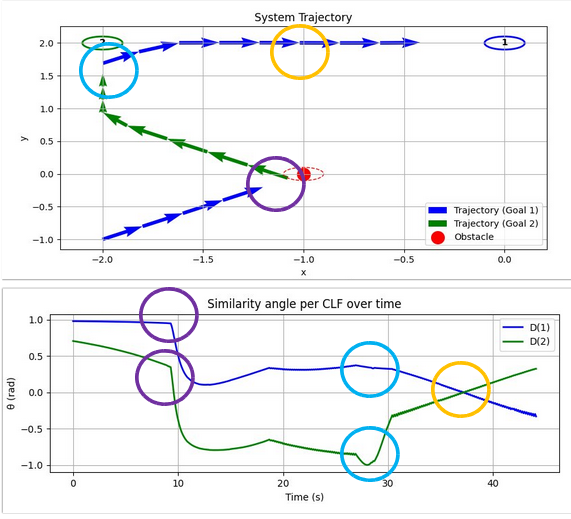}
    \caption{\textbf{Consistent alignment-based heuristic.} Trajectories and the corresponding plot of $\phi_{\theta}$ are color-coded by the engaged goal. For the engaged goal 1 (blue curve), $\phi_{\theta} \approx 1$ when conflict is maximal; the obstacle lies in the line-of-sight between the system and goal 1. For the engaged goal 2 (green curve), $\phi_{\theta} \approx -1$ when conflict is minimal; the obstacle lies directly the system while it approaches goal 2.}
    \label{fig:collinearity_cos_sim_diff_cos_sim}
\end{figure}

Relying on directional information alone is too conservative, as a conflict can be detected from anywhere in the state space. To localize conflict detection, we use Guassian scaling to ensure a smooth decay of conflict values away from the safety boundary:

\begin{equation}
\begin{aligned}
    \phi'_{\theta}(\mathbf{x}, V(\mathbf{x}), h(\mathbf{x})) &= f(h(\mathbf{x}), \sigma) \; \phi_{\theta}(\mathbf{x}, V(\mathbf{x}), h(\mathbf{x})), \\[6pt]
    \text{where} \quad f(h(\mathbf{x}), \sigma) &= \frac{1}{\sigma\sqrt{2\pi}} \exp\left(-\frac{h(\mathbf{x})^2}{2\sigma^2}\right).
    \label{eq:cos_sim_scaled}
\end{aligned}
\end{equation}
Here, $\sigma$ dictates the proximity of activation of the conflict with respect to the safety boundary imposed by $h$. This makes $\phi'_{\theta}$ \textit{locally consistent}:

\begin{proposition}[Local Consistency of $\phi'_\theta$]
The scaled heuristic $\phi'_\theta$ is consistent in a neighborhood $\mathcal{X}_h \subseteq \mathcal{X}$ of the safety boundary $\partial\mathcal{S} = \{\mathbf{x} \in \mathcal{X} : h(\mathbf{x}) = 0\}$. That is, for any $V_1, V_2 \in \mathcal{V}$, $h \in \mathcal{H}$:
\begin{equation}
    \phi'_\theta(\mathbf{x}, V_1, h) \leq \phi'_\theta(\mathbf{x}, V_2, h) \iff 
    \phi_\theta(\mathbf{x}, V_1, h) \leq \phi_\theta(\mathbf{x}, V_2, h)
    \label{eq:local_consistency}
\end{equation}
for all $\mathbf{x} \in \mathcal{X}_h = \{\mathbf{x} \in \mathcal{X} : f(h(\mathbf{x}), \sigma) \geq \epsilon\}$ for some threshold $\epsilon > 0$ determined by $\sigma$.
\label{prop:local_consistency}
\end{proposition}

\begin{proof}
Since $f(h(\mathbf{x}), \sigma) > 0$ for all $\mathbf{x} \in \mathcal{X}_h$, multiplication by $f(h(\mathbf{x}), \sigma)$ preserves the ordering of $\phi_\theta$ within $\mathcal{X}_h$ for consistent heuristic condition \eqref{eq:consistent_heuristic}. Hence, consistency follows directly from Definition~\ref{def:consistent_heuristic}. Outside $\mathcal{X}_h$,  the ordering is no longer meaningful since $\phi'_\theta \to 0$ for all $V_i \in \mathcal{V}$.
\end{proof}

Based on this heuristic, we define a policy $\pi_{\theta}$ that selects the least conflicted objectives for a given $\mathcal{V}(t)$ and $h$ as:

\begin{equation}
    \mathcal{V}^*(t) = \pi_{\theta}(\mathbf{x}, \mathcal{V}(t)) \big|_{h} = \argmin_{V_i \in \mathcal{V}(t)} \; \phi'_{\theta}(\mathbf{x}, V_i, h(\mathbf{x})) \\
    \label{eq:policy_cos_sim_scaled}
\end{equation}
where $\mathcal{V}^*$ is the set of CLF values associated with the least equally-conflicted goals at current time $t$. If $|\mathcal{V}^*(t)| > 1$, a unique CLF $V^*$ is selected from  $\mathcal{V}^*(t)$ via policy \eqref{eq:policy_min_clf}. To trigger switching while avoiding Zeno behavior \cite{zeno_sufficient_conditions}, we establish the following switching condition: 
\begin{equation}
    \left(t - t^{\text{sw}}_i \geq \tau\right) \wedge \left(\phi'_{\theta}(\mathbf{x}, V^*, h) \geq \epsilon_{\theta}\right)
    \label{eq:cos_sim_zeno_switching}
\end{equation}
where $t^{sw}_i$ is the last switch time, $\tau$ is the dwell-time threshold, and $\epsilon_{\theta}$ is the threshold for cosine-similarity angle between $\nabla V$ and $\nabla h$, as determined by \eqref{eq:cos_sim_scaled}.

\section{Experimental Results}

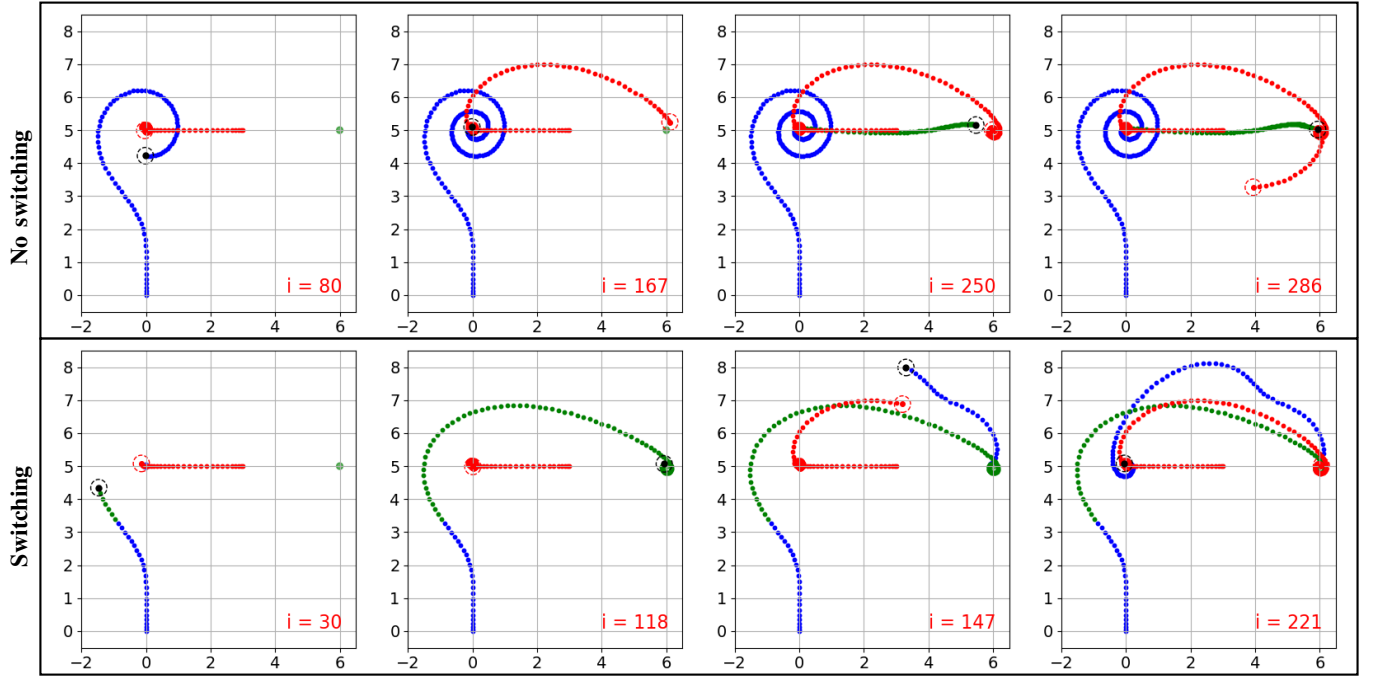
\begin{figure*}[t]
\begin{tikzpicture}

\def\framewidth{0.98\textwidth}
\def\frameheight{0.25\textwidth}

\begin{scope}[shift={(0, \frameheight)}]
  \node[anchor=south west, inner sep=0] at (0,0) {\includegraphics[width=\framewidth, height=\frameheight]{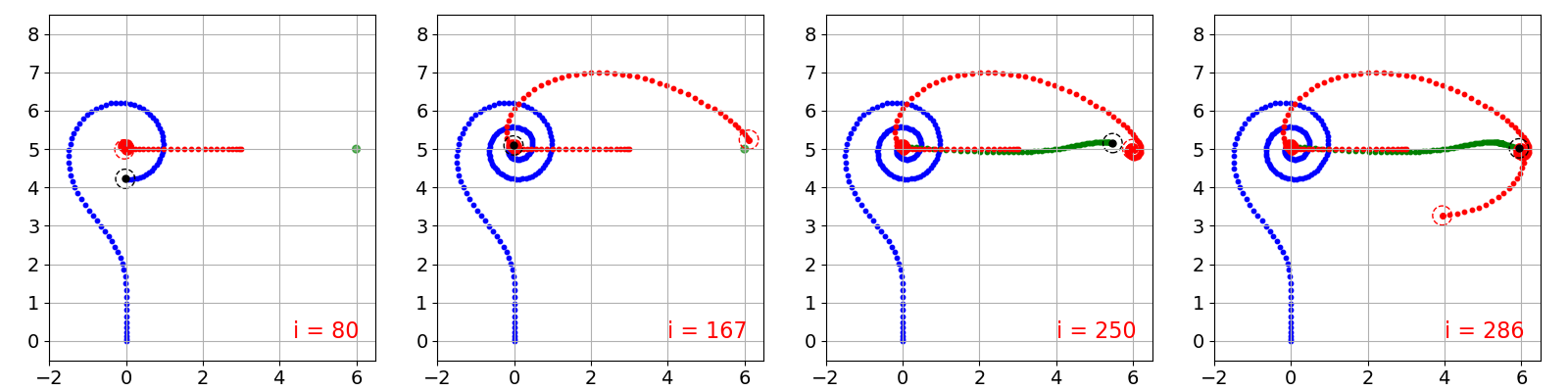}};
  \draw[thick] (0,0) rectangle (\framewidth, \frameheight);
\end{scope}

\begin{scope}[shift={(0, 0)}]  
  \node[anchor=south west, inner sep=0] at (0,0) {\includegraphics[width=\framewidth, height=\frameheight]{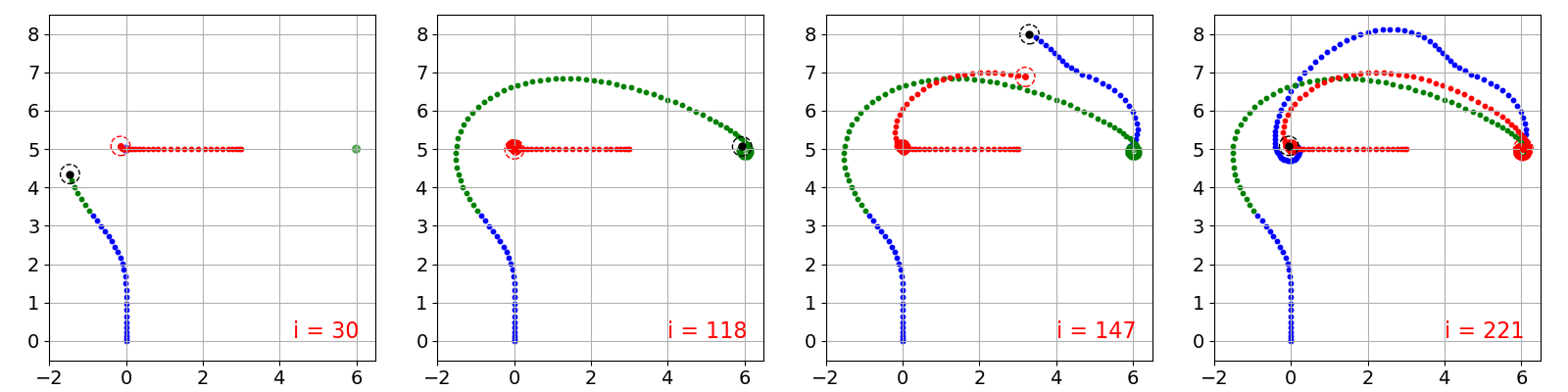}};
  \draw[thick] (0,0) rectangle (\framewidth, \frameheight);
\end{scope}

\node[anchor=east, font=\small\bfseries, rotate=90] at (-0.25, 0.65*\frameheight + \frameheight) {No switching};
\node[anchor=east, font=\small\bfseries, rotate=90] at (-0.25, 0.65*\frameheight) {Switching};

\end{tikzpicture}
\caption{\textbf{Patrolling scenario with unicycle dynamics.} A single agent must reach goals 1 (blue) and 2 (green) in sequence while avoiding a moving obstacle ``patrolling" the goals. \textbf{No switching (top):} The agent waits if the obstacle occupies a goal and proceeds once it is clear. \textbf{Switching (bottom):} The agent detects conflicts and redirects dynamically to avoid delays. Simulation parameters are given in Table~\ref{tab:sim_params_patrolling}.}
\label{fig:patrolling}
\end{figure*}

We evaluate the proposed conflict-aware switching policy \eqref{eq:policy_cos_sim_scaled} 
against the baseline sequential policy \eqref{eq:policy_next_clf} by measuring the total completion time (iterations) required to reach all goals. Performance is compared across three scenarios of increasing complexity:

\begin{enumerate}
    \item \textbf{Stationary obstacle:} A single agent navigating to two goals in the presence 
    of a stationary conflict-inducing obstacle. We vary the obstacle size to evaluate performance of our strategy against the baseline.
    \item \textbf{Patrolling obstacle:} A single agent operating under unicycle dynamics navigating to two goals guarded by a moving obstacle. 
    \item \textbf{Monte Carlo:} Multiple agents tasked with reaching their assigned goals while avoiding each other. We vary the numbers of goals and agents to evaluate scalability.
\end{enumerate}

\subsection{Scenario 1: Stationary Obstacle}
Consider a 2D driftless single-integrator of the form:
\begin{equation}
    \dot{\mathbf{x}} = 0.01 \begin{bmatrix} 1 & 0 \\ 0 & 1 \end{bmatrix} \mathbf{u}
    \label{eq:system}
\end{equation}
The following  CLF enforces a nominal control objective for goal $\mathbf{x}_{g}$:
\begin{equation}
    V(\mathbf{x}, \mathbf{x}_\mathfrak{g}) = \|\mathbf{x} - \mathbf{x}_\mathfrak{g}\|^2.
    \label{eq:vanilla_clf_1st_order}
\end{equation}
while safety constraint is enforced by a CBF of the form:
\begin{equation}
    h(\mathbf{x}) = \|\mathbf{x} - \mathbf{x}_{\text{obs}}\|^2 - r_{\text{safe}}^2.
    \label{eq:vanilla_cbf_first_order}
\end{equation}
Here $h(\mathbf{x})$ models a circular obstacle of radius $r_{\text{safe}}$ centered at $\mathbf{x}_{\text{obs}}$. The environment configuration is shown in Fig. \ref{fig:barrier_switching_success}.

\begin{figure} [H]
    \centering
    \begin{subfigure}{0.49\linewidth}
        \centering
        \includegraphics[width=\linewidth, height=\linewidth]{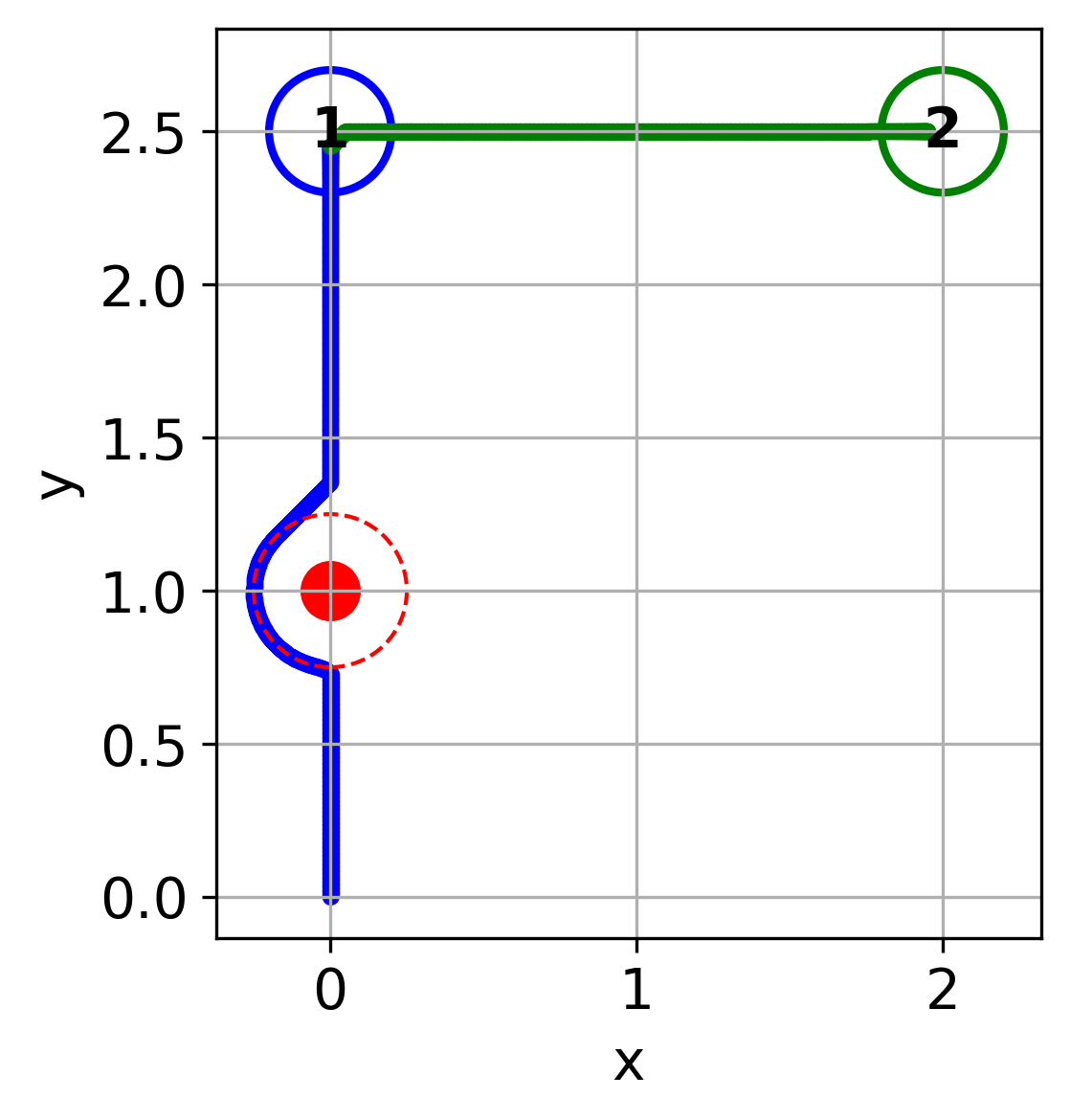}
        \caption{No switching}
        \captionsetup{justification=centering}
        \label{fig:barrier_switching_success_no_switch}
    \end{subfigure}
    \hfill
    \begin{subfigure}{0.49\linewidth}
        \centering
        \includegraphics[width=\linewidth, height=\linewidth]{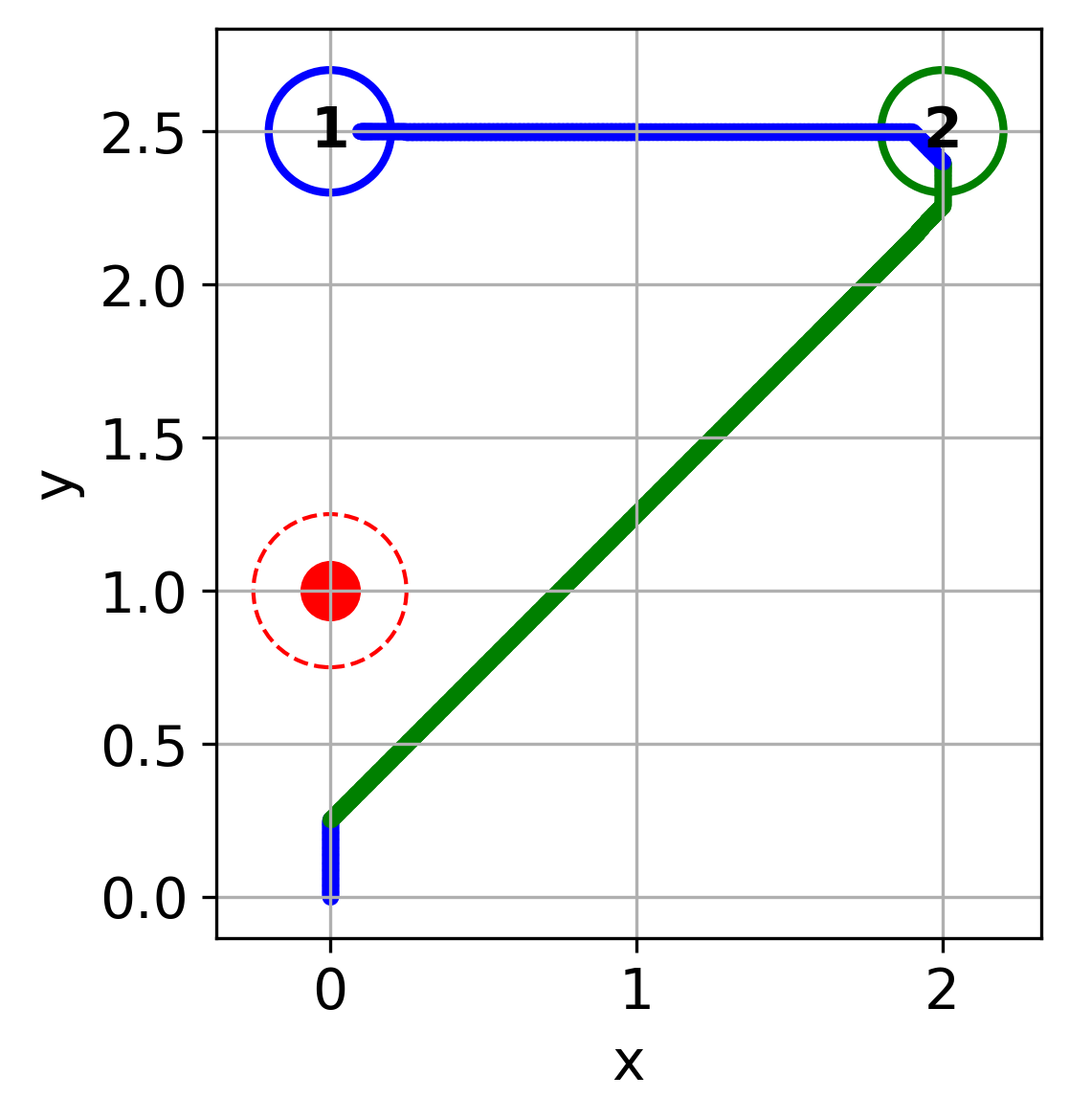}
        \caption{Switching}
        \captionsetup{justification=centering}
        \label{fig:barrier_switching_success_switch}
    \end{subfigure}
    \caption{\textbf{Stationary obstacle scenario:} A system is tasked with reaching goals $\mathcal{G} = \{\mathfrak{g}_1, \mathfrak{g}_2\}$. An obstacle induces a conflict with respect to the system's starting position at $g_1$. See simulation parameters in Table \ref{tab:sim_params_single_integrator}.}
    \label{fig:barrier_switching_success}
\end{figure}

The agent initially approaches goal 1, where the obstacle induces a conflict. Baseline switching policy $\pi_{seq}$ forces the agent to navigate around the obstacle. To increase conflict, we gradualy inflate the obstacle size gradually by increasing $r_{\text{safe}}$. As shown in Fig. \ref{fig:single_int_switching_trend}, $\pi_{seq}$ is faster for a very small obstacle size since switching to goal 2 would result in a longer overall trajectory. However, the total completion time increases linearly with the obstacle size due to the slowdown induced near the obstacle, as the agent is forced to accommodate conflicting nominal and conflicting constraints. Conversely, completion time decreases with $\phi'_{theta}$ as obstacle size increases: a larger $r_{safe}$ value increases the activation neighborhood (see Prop. \ref{prop:local_consistency}) for triggering an earlier switch for larger values of $h$.

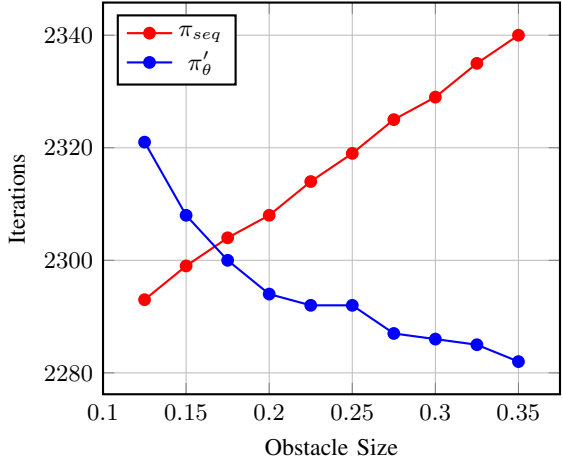
\begin{figure}[H]
\centering
\begin{tikzpicture}
\begin{axis}[
    width=0.7\linewidth,
    height=0.6\linewidth,
    xlabel={Obstacle Size},
    ylabel={Iterations},
    grid=both,
    legend pos=north west,
    thick,
    xmin=0.1,
    xmax=0.375,
    xtick={0.0,0.05,0.10,0.15,0.20,0.25,0.30,0.35},
    xticklabel style={/pgf/number format/fixed},
    yticklabel style={/pgf/number format/1000 sep={}},
    scale only axis,
    font=\figfontsize,
    tick label style={font=\figfontsize},
    label style={font=\figfontsize},
    legend style={font=\figfontsize}
    ]

\addplot[
    color=red,
    mark=*,
    solid
] coordinates {
 (0.125, 2293)
 (0.15, 2299)
 (0.175, 2304)
 (0.2, 2308)
 (0.225, 2314)
 (0.25, 2319)
 (0.275, 2325)
 (0.3, 2329)
 (0.325, 2335)
 (0.35, 2340)
};

\addplot[
    color=blue,
    mark=*,
    solid
] coordinates {
 (0.125, 2321)
 (0.15, 2308)
 (0.175, 2300)
 (0.2, 2294)
 (0.225, 2292)
 (0.25, 2292)
 (0.275, 2287)
 (0.3, 2286)
 (0.325, 2285)
 (0.35, 2282)
};

\legend{$\pi_{seq}$, $\pi'_{\theta}$}

\end{axis}
\end{tikzpicture}
\caption{Comparison of No-Switching and Barrier-Switch as functions of obstacle size. See Table \ref{tab:sim_params_single_integrator} for a list of simulation parameters.}
\label{fig:single_int_switching_trend}
\end{figure}

\subsection{Patrolling Obstacle}

Consider a unicycle model with state $\mathbf{x} = [x, y, v, \theta]^T \in \mathbb{R}^4$, where $x, y$ denote the planar position, $v$ the linear velocity, and $\theta$ is the heading angle. The control input is $u = [a, \omega]^T \in \mathbb{R}^2$, where $a$ and $\omega$ denote the linear acceleration and angular velocity, respectively. The control-affine dynamics for the agent and obstacle are given by:
\begin{equation}
    \dot{\mathbf{x}} = f(\mathbf{x}) + g(\mathbf{x})\mathbf{u}
\end{equation}
where the drift and control matrix are:
\begin{equation}
    f(\mathbf{x}) = \begin{bmatrix} v\cos\theta \\ v\sin\theta \\ 0 \\ 0 \end{bmatrix}, \quad g(\mathbf{x}) = \begin{bmatrix} 0 & 0 \\ 0 & 0 \\ 1 & 0 \\ 0 & 1 \end{bmatrix}
\end{equation}
To measure nominal control objectives, we use the first order CLF as in \eqref{eq:vanilla_cbf_first_order}. However, due to a relative order of 2 with respect to the dynamics \cite{HOCLFs}, we use a proportional controller adapted from \cite{cbfkit} to enforce nominal control objectives. Given the agent state $\mathbf{x} = (x,\ y,\ v,\ \theta)$ and a goal position $\mathbf{g} \in \mathbb{R}^2$, the controller first computes the position error and desired heading angle:
\begin{align}
    \mathbf{e_{pos}} &= [e_x, e_y]^T = \mathbf{g} - [x, y]^T, \\
    \theta_d &= \arctan2(e_y,\ e_x).
\end{align}
The heading error is wrapped to $[-\pi, \pi]$ to ensure the shortest rotational path is taken:
\begin{equation}
    \theta_e = \left(\left[(\theta_d - \theta) + \pi \right] \bmod 2\pi\right) - \pi.
\end{equation}
The desired speed is set proportionally to distance to the goal $d = \|\mathbf{e}\|$, 
and clipped to a maximum of $2.0$ m/s:
\begin{equation}
    v_d = \min\left(K_p \cdot d,\ 2.0\right).
\end{equation}
Linear acceleration $a$ and angular velocity $\omega$ are then calculated as:
\begin{equation}
    \mathbf{u}_{\text{nom}} = \begin{bmatrix} a \\ \omega \end{bmatrix} = 
    \begin{bmatrix} K_p & 0 \\ 0 & K_\theta \end{bmatrix}
    \begin{bmatrix} v_d - v \\ \theta_e \end{bmatrix}
    \label{eq:u_nom}
\end{equation}
with proportional gains $K_p = 1.0$ and $K_\theta = 0.5$.

Since the first order control barrier function in \eqref{eq:vanilla_cbf_first_order} has a relative degree of 2 with respect to the unicycle model \cite{ECBFs}, safety is enforced by respecting second order CBF constraints of the form:
\begin{align}
    L_g L_f h\, \mathbf{u} \;\geq\;
    &  -k\cdot\begin{bmatrix} \zeta_1 & \zeta_2 \end{bmatrix} \begin{bmatrix} L_f h \\[4pt] h \end{bmatrix} - L_f^2 h
    \label{eq:cbf_rel_constraint}
\end{align}
where $\zeta_1=1.0$ and $\zeta_2=0.75$. $k=10.0$ is a linear gain. The optimal control input for the agent is given by:
\begin{equation}
\label{eq:qp_patrolling_agent}
\begin{aligned}
\mathbf{u}^{agent}_{\mathrm{opt}}(\mathbf{x})
= &\argmin_{(\mathbf{u},\rho)\in\mathbb{R}^{m+1}} 
&& \frac{1}{2} (\mathbf{u}-\mathbf{u_{nom}})^T(\mathbf{u}-\mathbf{u_{nom}}) + 2s\rho^2 \\
& \quad \text{s.t. } &&\eqref{eq:cbf_rel_constraint} \text{ holds for } h\\
&&& -\mathbf{u}_{\max} \le \mathbf{u} \le \mathbf{u}_{\max},\\
&&& \rho \ge 0.
\end{aligned}
\end{equation}
where $u_{nom}$ is calculated based on the currently selected goal. Here  $\rho$ is a slack variable which relaxes the nominal control objective to prioritize safety, and $s=50.0$ is the slack penalty value. The obstacle does not observe any safety constraints, and keeps switching between both goals by following $\pi_{seq}$ indefinitely. Therefore, the control input for the obstacle is calculated from \eqref{eq:u_nom} based on its selected goal. To force the obstacle to spend more time near goals, we enforce more restrictive control bounds and a tighter goal completion threshold for the obstacle as compared to the agent (see Table \ref{tab:sim_params_patrolling}).

Fig. \ref{fig:patrolling} shows a moving obstacle guarding two goals. The obstacle imposes a safety constraint on the agent. When the agent is operating under $\pi_{seq}$, it ends up waiting for the obstacle to clear the goal at both goal locations. In comparison, $\phi'_{\theta}$ allows the agent to switch to the alternate goal, and achieve a lower total completion time. This shows the compatibility of \eqref{eq:cos_sim_scaled} with high-order control constraints, where \eqref{eq:reis_collinearity_heuristic_original} would fail. 

\subsection{Multi-Agent-Multi-Goal Navigation}

To study the scalability of a switching-based reach-and-avoid approach, we consider a multi-agent-multi-goal scenario where a set of $r$ robots $\mathcal{R} = \{x_1, \dots, x_r\}$ navigate to $k$ goals $\mathcal{G} = \{\mathfrak{g}_1, \dots, \mathfrak{g}_k\}$ while avoiding each other, as shown in Algorithm~\ref{alg:multi_agent_multi_goal}. Each robots starts with an initial set of goal assignments $\mathcal{G}_i$. The objective of this scenario is for each agent to reach the terminal state $\mathcal{V}_i = \emptyset$. A safety constraint $h_{ij} \in \mathcal{H}_i$ imposes safety for agent $j$ with respect to agent $i$, where $\mathcal{H}_i$ is the set of all safety constraints observed by agent $i$. When agent $i$ completes all its objectives ($\mathcal{G}_i = \emptyset$), it remains at it's last-satisfied goal, and it's corresponding safety constraint is removed from consideration for switching for all remaining agents. 



    

\begin{algorithm} [h]
\caption{Mult-agent-multi-goal switching}
\begin{algorithmic}[1] 
\Require $\mathcal{R} = \{x_1, \dots, x_r\}$, $V$, $h$, $\pi$
\Require $\mathcal{G} = \{\mathcal{G}_1, \dots, \mathcal{G}_r\}$, \; $\mathcal{G}_i = \{\mathfrak{g}_{i,1}, \dots, \mathfrak{g}_{i,k}\}$ $\forall a_i \in \mathcal{R}$
\Ensure  $\mathcal{G} = \emptyset$ 
\For{$x_i \in \mathcal{R}$}
    \State $\mathcal{V}_i(0) \gets \{V_{i,1}, \dots, V_{i,k}\}$
    \State $\mathcal{H}_i(0) \gets \{h_{ij} : x_j \in \mathcal{R} \setminus \{x_i\}\}$
    \State $t_i^{\text{sw}} \gets -\tau$
    \State $V_i^* \gets \pi_{\text{clf}}(x_i, \mathcal{V}_i(0))$
\EndFor
\While{$\mathcal{G} \neq \emptyset$}
    \For{$x_i \in \mathcal{R}$}
       \If{$\mathcal{V}_i(t) = \emptyset$} \Comment{all goals completed}
            \State $\mathcal{H}_j(t) \gets \mathcal{H}_j(t) \setminus \{h_{ji}\}, \quad \forall a_j \in \mathcal{R} \setminus \{a_i\}$
            \State $u_i \gets 0$
        \Else
            \For{$\mathfrak{g}_{i, k} \in \mathcal{G}_i$} \Comment{check goal completion}
                \If{$\|x_i - \mathfrak{g}_{i, k}\|^2 \leq \epsilon_\mathfrak{g}$}
                    \State $\mathcal{V}_i(t) \gets \mathcal{V}_i(t) \setminus \{V_{i, k}\}$
                    \State $\mathcal{G}_i(t) \gets \mathcal{G}_i(t) \setminus \{\mathfrak{g}_{i, k}\}$
                \EndIf 
            \EndFor
            \If{$\mathcal{H}_i(t) \neq \emptyset$} \Comment{if other agents are active}
               \State $i_{min} = \argmin_{h \in \mathcal{H}_i(t)} \pi'_{\theta}(x, \mathcal{V}(t)) \big|_{h}$
            \Else
                \State $i_{min} = \pi_{\text{clf}}(x, \mathcal{V}(t))$
            \EndIf
            \If{$\mathfrak{g}_i^* \neq \mathfrak{g}_{i_{min}}$}
                \If{$(t - t_i^{\text{sw}} \geq \tau)$}
                \State $\mathfrak{g}_i^* = \mathfrak{g}_{i_{min}}$
                \State $V_i^* = V_{i_{min}}$
                \State $t_i^{\text{sw}} \gets t$
            \EndIf
            \EndIf
            \State $u_i \gets \mathbf{u}^*(x_i, V_i^*, \mathcal{H}_i(0))$\Comment{solve QP \eqref{eq:ncbf_nclf_qp}}
        \EndIf
    \EndFor
\EndWhile
\end{algorithmic}
\label{alg:multi_agent_multi_goal}
\end{algorithm}

\begin{figure}[h]
    \centering
        \begin{tikzpicture}
            \begin{axis}[
                width=0.8\linewidth,
                height=0.6\linewidth,
                xlabel={Scenario},
                ylabel={Timeout (\%)},
                xtick={1,2,3,4,5},
                xticklabels={
                    \shortstack{6 Agents\\3 Goals},
                    \shortstack{9 Agents\\6 Goals},
                    \shortstack{12 Agents\\7 Goals},
                    \shortstack{15 Agents\\10 Goals},
                    \shortstack{19 Agents\\14 Goals}
                },
                grid=both,
                legend pos=north west,
                thick,
                yticklabel style={/pgf/number format/fixed},
                scale only axis,
                font=\figfontsize,
                tick label style={font=\figfontsize},
                label style={font=\figfontsize},
                legend style={font=\figfontsize}
            ]
            
            \addplot[
                color=red,
                mark=*,
                solid
            ] coordinates {
            (1,20.8)
            (2,25.2)
            (3,35.4)
            (4,41.0)
            (5,51.0)
            };
            
            \addplot[
                color=blue,
                mark=*,
                solid
            ] coordinates {
            (1,12.6)
            (2,26.4)
            (3,32.0)
            (4,33.4)
            (5,33.0)
            };
            
            \legend{$\pi_{seq}$, $\pi'_{\theta}$}
            \end{axis}
        \end{tikzpicture}
    \caption{Timeout percentage across scenarios.}
    \label{fig:single_int_monte_carlo_timeout_percentage}
\end{figure}
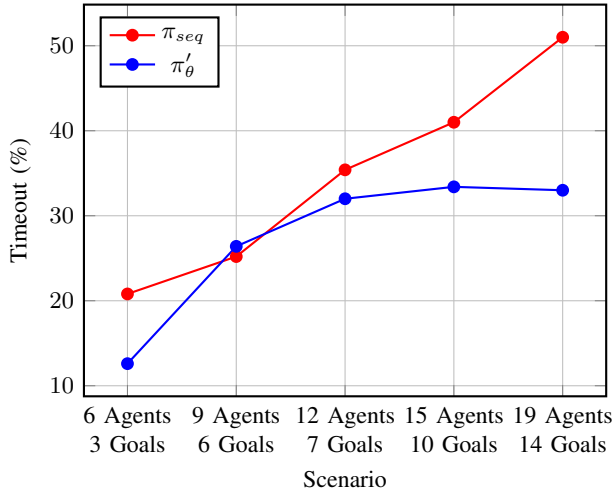

\begin{figure}[h]
\centering
    \begin{tikzpicture}
        \begin{axis}[
            width=0.8\linewidth,
            height=0.9\linewidth,
            xlabel={Scenario},
            ylabel={Average Completion Time},
            xtick={1,2,3,4,5},
            xticklabels={
                \shortstack{6 Agents\\3 Goals},
                \shortstack{9 Agents\\6 Goals},
                \shortstack{12 Agents\\7 Goals},
                \shortstack{15 Agents\\10 Goals},
                \shortstack{19 Agents\\14 Goals}
            },
            grid=both,
            legend pos=north west,
            thick,
            yticklabel style={/pgf/number format/1000 sep={}},
            scale only axis,
            font=\figfontsize,
            tick label style={font=\figfontsize},
            label style={font=\figfontsize},
            legend style={font=\figfontsize}
        ]
        
        \addplot[
            color=red,
            mark=*,
            solid
        ] coordinates {
        (1,2294.756)
        (2,2842.158)
        (3,3933.484)
        (4,4572.04)
        (5,5616.236)
        };
        
        \addplot[
            color=blue,
            mark=*,
            solid
        ] coordinates {
        (1,1528.988)
        (2,3046.512)
        (3,3748.55)
        (4,4066.132)
        (5,4259.386)
        };
        
        \addplot[
            color=red,
            mark=*,
            dashed
        ] coordinates {
        (1,271.16)
        (2,430.69)
        (3,609.11)
        (4,800.07)
        (5,1053.54)
        };
        
        \addplot[
            color=blue,
            mark=*,
            dashed
        ] coordinates {
        (1,307.77)
        (2,552.33)
        (3,806.69)
        (4,1090.29)
        (5,1431.92)
        };
        
        \legend{
        $\pi_{seq}$ (with timeouts),
        $\pi'_{\theta}$ (with timeouts),
        $\pi_{seq}$ (no timeouts),
        $\pi'_{\theta}$ (no timeouts)
        }
        
        \end{axis}
    \end{tikzpicture}
    \caption{Average completion time with and without timeouts.}
    \label{fig:single_int_monte_carlo_completion_time}
\end{figure}
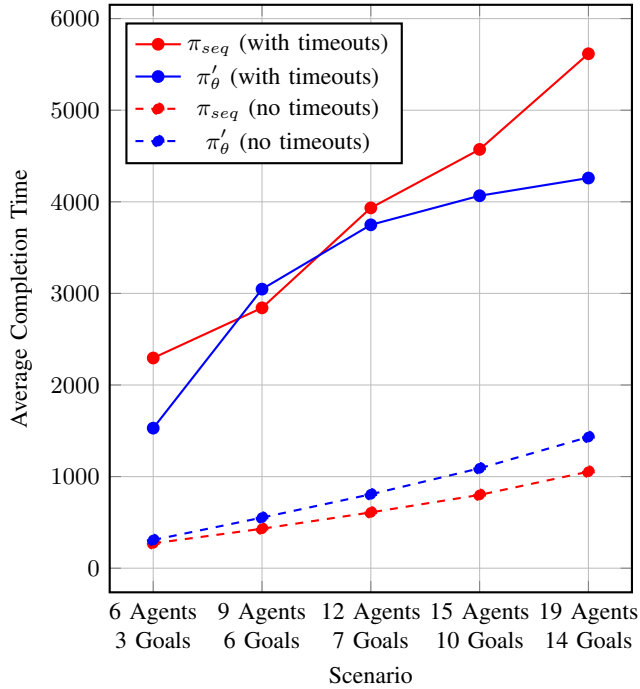

We increase the conflict by increasing the number of agents and goals within a fixed simulation space in $\mathbb{R}^2$, from 6 agents with 3 goals assigned to each agent, to 19 agents with 14 goals. Agent/goal locations and goal assignments are randomized over 500 monte-carlo runs for each case. We limit the iterations to 10,000. As shown in Fig. \ref{fig:single_int_monte_carlo_completion_time}, $\pi'_{\theta}$-based switching is able to achieve a much lower average completion time across all 500 runs when considering timed-out experiments, whereas $\pi_{seq}$ leads to higher completion times with increasing agents/goals. This trend is also observed in Fig. \ref{fig:single_int_monte_carlo_timeout_percentage}, where the timeout percentage plateaus with $\pi'_{\theta}$-based switching, and increases monotonically for $\approx 33.4 \%$. This indicates the ability of a consistent-heuristic based switching policy to perform better without fine-tuning for each spatial configuration. Completion times reported without timeouts indicate scenarios where a sequential goal-switching strategy occasionally performs better due to particular spatial configurations.

\section{Conclusion}

In this work, we highlighted the need for adopting an optimal and complete switching strategy for multi-goal reach-and-avoid problems using CLF-CBF QPs to increase performance with respect to completion time for satisfying all goals. We established the conditions under which a heuristic and corresponding switching policy can be used to detect and switch to the lower-conflicted goal, whenever one is available. We identified a consistent alignment-based heuristic. Using this heuristic, we demonstrated that a switching policy can reduce completion time as conflict increases under three multi-goal navigation scenarios of increasing complexity: single agent avoiding a stationary obstacle, a single agent avoiding a moving obstacle, and a multi-agent-multi-goal navigation scenario.

\bibliographystyle{IEEEtran}
\bibliography{bibliography}

\appendix

\begin{table}[h]
\centering
\caption{Sim Params: Single Integrator Switching Comparison}
\label{tab:sim_params_single_integrator}
\begin{tabular}{lll}
\hline
\textbf{Parameter} & \textbf{Value} & \textbf{Description} \\
\hline
$\Delta t$ & $0.01$ s & Time step \\
$N$ & $2500$ & Total iterations \\
$u_{\max}$ & $(0.2,\ 0.2)$ & Maximum control input \\
$r_{\text{safe}}$ & $0.1$--$0.35$ & CBF safety radius \\
$\tau_{\text{dwell}}$ & $8$ s / 800 iter & Minimum dwell time \\
$\epsilon_{\theta}$ & $0.1$ & Similarity angle threshold \\
$\epsilon_{\text{goal}}$ & $0.1$ & Goal completion threshold \\
$p$ & $1000.0$ & Slack penalty \\
$\alpha$ & $10.0$ & CBF gain \\
$\gamma$ & $2.5$ & CLF gain \\
\hline
\end{tabular}
\end{table}

\begin{table}[h]
\centering
\caption{Sim Params: Unicycle Dynamics Patrolling Scenario}
\begin{tabular}{lll}
\hline
\textbf{Parameter} & \textbf{Value} & \textbf{Description} \\
\hline
$\Delta t$ & $0.1$ s & Time step \\
$N$ & $500$ & Total iterations \\
$u_{\max, agent}$ & $(5.0,\ 5.0)$ & Maximum control input (agent) \\
$u_{\max, obs}$ & $(2.5,\ 25.0)$ & Maximum control input (obstacle) \\
$r_{\text{safe}}$ & $0.5$ & CBF safety radius \\
$\tau_{\text{dwell}}$ & $4.0$ s / 40 iter & Minimum dwell time \\
$\epsilon_{\text{goal, agent}}$ & $0.1$ & Agent goal completion threshold \\
$\epsilon_{\text{goal, obs}}$ & $0.01$ & Obstacle goal completion threshold \\
$\epsilon_{\theta}$ & $0.1$ & Similarity angle threshold \\
$p$ & $10$ & Slack penalty \\
$\alpha$ & $10$ & CBF gain \\
$\gamma$ & $1.0$ & CLF gain \\
\hline
\end{tabular}
\label{tab:sim_params_patrolling}
\end{table}

\begin{table}[h]
\centering
\caption{Sim Params: Single Integrator Monte Carlo Experiments}
\begin{tabular}{lll}
\hline
\textbf{Parameter} & \textbf{Value} & \textbf{Description} \\
\hline
$\Delta t$ & $0.1$ s & Time step \\
$N$ & $2000$ & Maximum iterations \\
$u_{\max}$ & $(0.25,\ 0.25)$ & Maximum control input \\
$r_{\text{safe}}$ & $0.1$ & CBF safety radius \\
$\tau_{\text{dwell}}$ & $4.0$ s / 40 iter & Minimum dwell time \\
$\epsilon_{\text{goal}}$ & $0.25$ & Goal completion threshold \\
$\epsilon_{\theta}$ & $0.1$ & Similarity angle threshold \\
$p$ & $1.0$ & Slack penalty \\
$\alpha$ & $1.0$ & CBF gain \\
$\gamma$ & $10.0$ & CLF gain \\
\hline
\end{tabular}
\label{tab:sim_params_single_int_monte_carlo}
\end{table}

\end{document}